\documentclass[10pt,letterpaper]{article}

\usepackage{amsmath,amssymb,amsthm, bm}

\usepackage{nameref,hyperref}

\usepackage{lastpage,graphicx}
\usepackage{epstopdf}
\usepackage{float}
\usepackage{fullpage}

\newtheorem{proposition}{Proposition}
\newtheorem{corollary}{Corollary}

\begin{document}
\vspace*{0.2in}

\begin{flushleft}
{\Large
\textbf\newline{Cross-Validation for Training and Testing Co-occurrence Network Inference Algorithms
} 
}
\newline
\\
Daniel Agyapong\textsuperscript{1},
Jeffrey Ryan Propster\textsuperscript{2},
Jane Marks\textsuperscript{2},
Toby Dylan Hocking\textsuperscript{1}
\\
\bigskip
\textbf{1} School of Informatics, Computing, and Cyber Systems, Northern Arizona University, Flagstaff, Arizona, United States of America
\\
\textbf{2} Department of Biological Sciences, Northern Arizona University, Flagstaff, Arizona, United States of America
\\
\bigskip

\end{flushleft}
\section*{Abstract}
Microorganisms are found in almost every environment, including the soil, water, air, and inside other organisms, like animals and plants. 
While some microorganisms cause diseases, most of them help in biological processes such as decomposition, fermentation and nutrient cycling. 
A lot of research has gone into studying microbial communities in various environments and how their interactions and relationships can provide insights into various diseases.
Co-occurrence network inference algorithms help us understand the complex associations of micro-organisms, especially bacteria. 
Existing network inference algorithms employ techniques such as correlation, regularized linear regression, and conditional dependence, which have different hyper-parameters that determine the sparsity of the network. 
Previous methods for evaluating the quality of the inferred network include using external data, and network consistency across sub-samples, both which have several drawbacks that limit their applicability in real microbiome composition data sets.
We propose a novel cross-validation method to evaluate co-occurrence network inference algorithms, and new methods for applying existing algorithms to predict on test data.
Our empirical study shows that the proposed method is useful for hyper-parameter selection (training) and comparing the quality of the inferred networks between different algorithms (testing).

\section*{Introduction}
Micro-organisms form complex ecological interactions such as mutualism, parasitism/predation, competition, commensalism and amensalism~\cite{doi:10.1128/mSystems.00124-19}. 
The human body hosts complex microbial communities consisting of bacteria, protozoa, archaea, viruses, and fungi.  
The human intestine alone has trillions of bacteria (microbiota), that have a symbiotic relationship with the host. 
The main function of the microbiota is to protect the intestine against colonization by harmful microorganisms like pathogens through mechanisms, such as competition for nutrients and modulation of host immune responses. 
Studying the interaction of the microbiota with pathogens and the host can offer new insights into disease pathogenesis and potential treatments~\cite{10.1111/j.1753-4887.2012.00493.x}. 
Over the past several years, the importance of the microbiome to human health and disease has become increasingly recognized. 
The trillions of microbes can protect us from colonization by pathogens, promote immunoregulation and tolerance by our own immune systems, and digest many of the foods that we ourselves cannot~\cite{BELKAID2014121}.
However, they can also contribute to disease, if their balance is disrupted by antibiotics, immune dysregulation, or other disturbances. 
The focus of this field has largely been on the bacterial members of the microbiome, since they make up the largest proportion of microbiota~\cite{doi:10.1126/science.aad9358}.
The bacteria exist alongside a diversity of organisms which can interact with each other and the host to impact health~\cite{Lloyd-Price2016}. 
Therefore, understanding the ecological interactions that occur in microbial communities is very crucial in maintaining a well-functioning ecosystem. 
To understand the interactions of microbial communities, it is beneficial to construct ecological networks that depict their positive and negative associations.
Numerous algorithms exist, each with their own set of hyper-parameters used to determine the level of sparsity (number of edges) in a network.

\paragraph{Summary of Contributions.}
Previous studies have primarily focused on evaluating co-occurrence network inference algorithms through external data, simulations, and network consistency across sub-samples.
Cross-validation has been used previously in training LASSO using validation error when used for prediction on validation set. 
Similarly, it has also been used in training Gaussian Graphical Models(GGM) for inferring a sparse inverse covariance matrix, achieved by maximizing the log likelihood of observed data under a Gaussian distribution with the estimated precision matrix.
In this paper, we present novel contributions that extend the existing research in this field.
We summarize the contributions of our paper, which enhance the current cross-validation technique, in Table 1~\ref{tab:sum_of_contrib}.
Firstly, we introduce new techniques for leveraging well-established algorithms such as Pearson/Spearman correlation and Gaussian Graphical Model for prediction on held-out or test data.
Secondly, we propose the utilization of prediction error on test set in cross-validation as a more widely applicable method for evaluating various algorithms on real microbiome data.
Lastly, we propose training the optimum correlation threshold in correlation based algorithms with cross-validation as compared to previous methods that use prior knowledge or pre-determined correlation threshold.


\begin{table}[H]
\centering
\caption{\textbf{Summary of contributions}}
\begin{tabular}{|p{3cm}|p{5cm}|p{3cm}|}
\hline
\textbf{Algorithm} & \textbf{Cross-validation } & \textbf{Cross-validation } \\ 
 & \textbf{ for training} & \textbf{ for testing} \\ \hline
LASSO & \texttt{CCLasso (2015)}~\cite{10.1093/bioinformatics/btv349}, \texttt{REBACCA (2015)}~\cite{10.1093/bioinformatics/btv364}, \texttt{SPIEC-EASI (2015)}~\cite{10.1371/journal.pcbi.1004226} & Proposed \\ \hline
GGM & \texttt{gCoda (2017)}~\cite{doi:10.1089/cmb.2017.0054}, \texttt{MDiNE (2019)}~\cite{10.1093/bioinformatics/btz824}, \texttt{COZINE (2020)}~\cite{Ha2020} & Proposed \\ \hline
Correlation (Pearson/Spearman) & Proposed & Proposed \\ \hline
\end{tabular}
\label{tab:sum_of_contrib}
\end{table}

\subsection*{Microbiome composition data sets}
There has been some challenges in obtaining microbiome abundance in different environments~\cite{Gilbert2014}. 
High-throughput Sequencing is used to sequence large amounts of DNA fragments at a relatively low cost~\cite{10.1371/journal.pone.0093827}. 
This involves amplifying a particular region of the bacterial genome through Polymerase Chain Reaction (PCR) and subsequently sequencing the produced amplicons. 
This region represents the 16S rRNA gene in bacteria, extensively employed as indicators for microbial classification and identification. 
The processed sequence are classified into Operational Taxonomic Units (OTU) with the aid of an advanced software which compares the sequences to reference database like the Ribosomal Database Project~\cite{10.1093/nar/gkt1244} and the Green Genes Database~\cite{doi:10.1128/AEM.03006-05}. 
Table~\ref{tab:datasets} presents some real microbiome composition data from public sources. 
Each Operational Taxonomic Unit (OTU) data describes the taxonomic composition of different samples from various environments. 
In Figure~\ref{fig:cv}, microbiome composition data set is represented by $N \times D$ matrix of counts (abundance) of bacteria, where each column represents a different type of bacteria (taxon) and each row represents a different sample.

\begin{table}[H]
\centering
\caption{{\bf Publicly Available Microbiome Composition Datasets}}
\begin{tabular}{|l|l|r|r|}
\hline

\bf Data & \bf Algorithm  & \bf Samples & \bf Taxa    \\ \hline

\href{https://github.com/tinglab/mLDM/blob/master/CRC/glne007.csv}{glne007} & \texttt{mLDM}~\cite{YANG2017129} & 490 & 338\\ \hline

\href{https://github.com/tinglab/mLDM/blob/master/CRC/Baxter_CRC.RData}{Baxter\_CRC} & \texttt{mLDM} ~\cite{YANG2017129} & 490 & 117\\ \hline

\href{https://github.com/zdk123/SpiecEasi/blob/master/data/amgut2.filt.phy.rda}{amgut2*} & \texttt{SPIEC\_EASI} ~\cite{10.1371/journal.pcbi.1004226} &   296 & 138\\ \hline

\href{https://github.com/zdk123/SpiecEasi/blob/master/data/amgut1.filt.rda}{amgut1*} &\texttt{SPIEC\_EASI} ~\cite{10.1371/journal.pcbi.1004226} &   289 & 127\\ \hline

\href{https://github.com/joey711/phyloseq/blob/master/data/enterotype.RData}{enterotype} & \texttt{phyloseq} ~\cite{10.1371/journal.pone.0061217} & 280 & 553\\ \hline

 \href{https://github.com/sahatava/MixMPLN/blob/master/data/real_data.csv}{MixMPLN\_real\_data} & \texttt{MixMPLN}~\cite{10.1093/bioinformatics/btz370} &  195 & 129\\ \hline

\href{https://github.com/kevinmcgregor/mdine/blob/master/data/crohns.RData}{crohns*} & \texttt{MDiNE} ~\cite{10.1093/bioinformatics/btz824} & 100 & 5\\ \hline

\href{https://github.com/MinJinHa/COZINE/blob/master/data/iOraldat.rda}{iOraldat*} & \texttt{COZINE} ~\cite{Ha2020} & 86 & 63\\ \hline

\href{https://github.com/joey711/phyloseq/blob/master/data/soilrep.RData}{soilrep} & \texttt{phyloseq} ~\cite{10.1371/journal.pone.0061217} & 56 & 16825\\ \hline

\href{https://github.com/zdk123/SpiecEasi/blob/master/data/hmp2.rda}{hmp216S} & \texttt{SPIEC\_EASI} ~\cite{10.1371/journal.pcbi.1004226} & 47 & 45\\ \hline

\href{https://github.com/zdk123/SpiecEasi/blob/master/data/hmp2.rda}{hmp2prot} & \texttt{SPIEC\_EASI} ~\cite{10.1371/journal.pcbi.1004226} & 47 & 43\\ \hline

\href{https://github.com/joey711/phyloseq/blob/master/data/esophagus.RData}{esophagus} & \texttt{phyloseq} ~\cite{10.1371/journal.pone.0061217} & 3 & 58\\ \hline

\end{tabular}
\label{tab:datasets}
\end{table}

\subsection*{Categorization of previous algorithms}
Many algorithms have been proposed to infer co-occurrence networks from real microbiome data sets.
In Table~\ref{tab:previous-inference-algorithms}, we grouped previous network inference algorithms into four categories: Pearson correlation (Pearson), Spearman correlation (Spearman), Least Absolute Shrinkage and Selection Operator (LASSO), and Gaussian Graphical Model (GGM). 
For example, SparCC~\cite{10.1371/journal.pcbi.1002687} estimates the Pearson correlations of log-transformed abundance data and uses an arbitrary threshold to limit the network, whereas MENAP~\cite{Deng2012} uses Random Matrix Theory to determine the correlation threshold of the standardized relative abundance data.
Both CCLasso~\cite{10.1093/bioinformatics/btv349} and REBACCA~\cite{10.1093/bioinformatics/btv364} employ LASSO to infer correlations among microbes using log-ratio transformed relative abundance data.
mLDM~\cite{YANG2017129} utilizes a graphical model to infer associations among microbes, as well as associations between microbes and environmental factors whereas SPIEC-EASI~\cite{10.1371/journal.pcbi.1004226} infers conditional dependencies among only microbes.

\subsection*{Previous sparsity hyper-parameter training methods} 
Each of the algorithms have their own set of hyper-parameters used to determine the level of sparsity (number of edges) in a co-occurrence network. 
For instance, in the Pearson and Spearman correlation inference algorithms, there is a threshold on the correlation coefficient which is typically chosen arbitrarily or using prior knowledge~\cite{Deng2012,10.1371/journal.pcbi.1002687, 10.12688/f1000research.9050.2}; edges with absolute coefficient magnitude below the threshold are removed from the network. 
The LASSO uses the degree of L1 regularization, typically selected using cross-validation to determine the sparsity of the network~\cite{10.1093/bioinformatics/btv349}.
The GGM infers the conditional dependencies between taxa by estimating the sparsity pattern of the precision matrix using penalized maximum likelihood methods through cross-validation~\cite{10.1371/journal.pcbi.1004226}.

\begin{table}[H]
\centering
\caption{
{\bf Previous microbial network inference algorithms}}
\label{tab:previous-inference-algorithms}
\begin{tabular}{|l|l|}
\hline

\multicolumn{1}{|l|}{\bf Method}  & \multicolumn{1}{|l|}{\bf Category} \\ \hline

\texttt{SparCC (2012)}~\cite{10.1371/journal.pcbi.1002687} & Pearson \\ \hline

\texttt{MENAP (2012)}~\cite{Deng2012} & 
Pearson, Spearman \\ \hline

\texttt{CoNet (2016)}~\cite{10.12688/f1000research.9050.2} & Pearson, Spearman \\ \hline

\texttt{CCLasso (2015)}~\cite{10.1093/bioinformatics/btv349} & LASSO \\ \hline

\texttt{REBACCA (2015)}~\cite{10.1093/bioinformatics/btv364} & GGM, LASSO \\ \hline

\texttt{SPIEC-EASI (2015)}~\cite{10.1371/journal.pcbi.1004226} & GGM, LASSO \\ \hline

\texttt{gCoda (2017)}~\cite{doi:10.1089/cmb.2017.0054} & GGM\\ \hline

\texttt{MDiNE (2019)}~\cite{10.1093/bioinformatics/btz824} & GGM \\ \hline

\texttt{HARMONIES (2020)}~\cite{10.3389/fgene.2020.00445} & GGM \\ \hline

\texttt{mLDM (2020)}~\cite{YANG2017129} & GGM \\ \hline

\texttt{COZINE (2020)}~\cite{Ha2020} & GGM \\ \hline

\texttt{PLNmodel (2021)}~\cite{10.3389/fevo.2021.588292} & GGM \\ \hline

\end{tabular}
\label{table1}
\end{table}

\subsection*{Previous evaluation criteria} 
Various evaluation criteria have been utilized to assess the performance of different algorithms used for network inference. 
Two of the most common ones are the use of external data sources and the network consistency across sub-samples. 
However, both criteria have some limitations. 
The use of external data sources may suffer from the lack of ground truth (external data may not be available or reliable) and the potential biases in the data. 
Some previous work that used this criterion are SparCC~\cite{10.1371/journal.pcbi.1002687} and SPIEC-EASI~\cite{10.1371/journal.pcbi.1004226}. 
The network consistency across sub-samples may favor trivial models that infer no associations (edges) and thus perfect consistency. 
An example of previous work that used this criterion is CCLASSO~\cite{10.1093/bioinformatics/btv349}.
Table \ref{tab:existing_evaluation} summarizes some existing microbial inference algorithms, methods compared, how they compare and evaluation type.

\begin {table} [H] 
\centering 
\caption {{\bf Existing Evaluation Methods}} 
\begin {tabular} {|l|p{0.5cm}|p{5.5cm}|p{2.5cm}|} 

\hline \multicolumn {1} {|p{2cm}|} {\raggedright \bf Algorithm} & \multicolumn {1} {p{2cm}|} {\raggedright \bf Algorithms compared} & \multicolumn {1} {p{5.5cm}|} {\raggedright \bf How they compare} & \multicolumn {1} {p{2.5cm}|} {\raggedright \bf Evaluation Type} \\ \hline

\texttt{SparCC(2012)} & \texttt{SparCC, Pearson} & Confusion matrix detected in the Pearson network by treating the SparCC network as the ground truth & External data (HMPOC dataset, build 1.0)~\cite{10.1371/journal.pone.0027310} \\ \hline

\texttt{REBACCA(2015)} & \texttt{REBACCA, SparCC, BP, ReBoot} & Consistency of positive and negative correlated taxonomic pairs identified independently from three data sets & External data (Mouse skin microbiota)~\cite{srinivas2013genome} \\ \hline

\texttt{SPIEC-EASI(2015)} & \texttt{SPIEC-EASIE, SparCC, CCREPE} & Consistency of sub-samples by measuring the Hamming distance between the hypothetical reference network and inferred network & External data (American Gut data set)~\cite{mcdonald2018american} \\ \hline
 
\texttt{CCLasso(2015)} & \texttt{CCLasso, SparCC} & Frobenius accuracy (measured by the Frobenius norm distance) between the estimated correlation matrices and a reference correlation matrix from data using half samples & Sub-sample analysis \\ \hline

\texttt{HARMONIES(2020)} & \texttt{HARMONIES, SPIEC-EASI, CCLasso, Pearson} & Accuracy of identifying true positive edges by comparing the estimated precision matrix with an arbitrarily chosen true one & External data \\ \hline

\texttt{mLDM(2020)} & \texttt{mLDM, SparCC, CClasso} & Power of association inference when compared to the reference association inference data~\cite{doi:10.1126/science.1262073} & External data (Tara Oceans Eukaryotic data) \\ \hline

\texttt{gCoda(2017)} & \texttt{gCoda, SPIEC-EASI} & False positive count on shuffled OTU data. & External data (Mouse Skin microbiome data)~\cite{srinivas2013genome} \\ \hline

\texttt{COZINE(2020)} & \texttt{COZINE, SPIEC-EASI, Ising} & The assortativity coefficient~\cite{PhysRevE.67.026126} (The likelihood of taxa existing within the same branch of the taxonomic tree to be interconnected within co-occurrence networks ) & External data (Oral microbiome data) \\ \hline

\end{tabular}
\label{tab:existing_evaluation}
\end{table}

\section*{Materials and Methods}
\subsection*{Preprocessing and Normalization of Dataset}
Microbial data sets are very high-dimensional in nature because they have substantial number of taxa that can be present in a single sample~\cite{doi:10.1128/mSystems.00016-19}. 
Their sparse nature makes even conventional machine learning algorithms struggle since they assume that most features are non-zero~\cite{hastie2009elements}.
Hence, it is crucial to apply appropriate preprocessing and normalization technique to convert the data set to a suitable format before conducting any data analysis~\cite{mayoue2012preprocessing}.
These are some of the notable methods for transforming sparse microbial data sets. \\



\subsubsection*{Standard Scaling}
Standard scaling normalizes each taxon column to have zero mean and unit variance for numerical stability.
This can help to reduce the influence of outliers and scale differences among taxa.
We apply standard scaling on the Amgut2 real microbiome data set and run the cross-validation analysis with various algorithms. 
The results are shown in Figure~\ref{fig:data_prep}. 
Let $N$ be the sample size, $\bar{x_j}$ be the mean of the $j^{th}$ taxa across all samples, $s_j$ be the standard deviation of the $j^{th}$ taxa across all samples and $x_{ij}$ be the count of taxon $j$ in sample $i\in\{1,\dots,N\}$.
The standard scaling transformation is given by :
\begin{equation*}
x_{ij} = \frac{x_{ij} - \bar{x_j}}{s_j}
\end{equation*}

\subsubsection*{Yeo-Johnson Power Transformation}
The Yeo-Johnson power transformation is a method for transforming numerical variables to approximate a normal distribution~\cite{10.1093/biomet/87.4.954}. 
This transformation is inspired by the log transformation that has been used in previous studies~\cite{10.1371/journal.pcbi.1002687, 10.1093/bioinformatics/btv349, 10.1093/bioinformatics/btv364}, but it differs in the mathematical function that it applies depending on the sign of the count value. Moreover, it involves a power parameter that determines the extent of the transformation and that is estimated from the data itself using the maximum likelihood method~\cite{10.1093/biomet/87.4.954}.
Let $\lambda$ be the power parameter, 
$x_{ij}$ be the count data and 
$y^{(\lambda)}$ be the transformed count data.
The Yeo-Johnson transformation is given by :
\begin{equation*} 
x_{ij}^{(\lambda)} = 
\begin{cases} 
\left[\left(y+1\right)^{\lambda}-1\right] / \lambda, & \text {if } \lambda \neq 0, y \geq 0 \\ \log \left(y+1\right), & \text {if } \lambda = 0, y \geq 0 \\ -\left[\left(-y+1\right)^{2-\lambda}-1\right] / (2-\lambda), & \text {if } 2-\lambda \neq 0, y < 0 \\ -\log \left(-y+1\right), & \text {if } 2-\lambda = 0, y < 0 
\end{cases} 
\end{equation*}
The Amgut2 real microbiome dataset undergoes the Yeo-Johnson transformation and standard scaling before we conduct the cross-validation analysis with various algorithms. Figure~\ref{fig:data_prep} plots the error on test set of the algorithms against the number of train samples. The results demonstrate that the Yeo-Johnson transformation substantially enhances the prediction accuracy on the test set relative to standard scaling only.

\begin{figure}[H]
\includegraphics[scale=0.7]{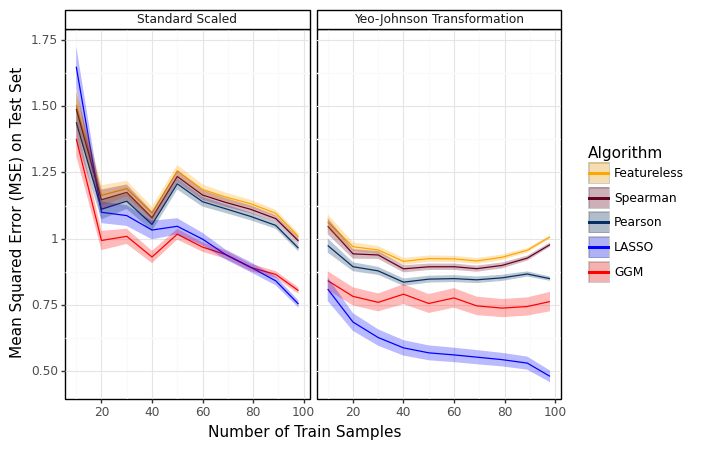}
\caption{\textbf{How different Transformation Methods affect the Prediction Accuracy of algorithms}\\
This figure evaluates the performance of different algorithms on the Amgut2 real dataset under standard scaling only (left panel) and, Yeo-Johnson transformation and standard scaling (right panel). 
The results imply that just standard scaling alone (applied to the raw data set) yields lower accuracy than the combination of Yeo-Johnson and standard scaling for each of the algorithms compared.}
\label{fig:data_prep}
\end{figure}

\subsection*{Cross-Validation for Evaluating Co-occurrence Network Inference Algorithms}
\label{sec:cv}
Cross-validation is a standard algorithm in machine learning used for selection, evaluation and estimation of performance of models. 
It has been previously used in the context of microbiome for training co-occurrence network inference algorithms~\cite{10.1093/bioinformatics/btv349}. 
Our study introduces cross-validation as a novel criterion to test the performance of co-occurrence network inference algorithms on microbiome data.
In Figure~\ref{fig:cv}A, we show how $K=3$ fold cross-validation can be used in the context of microbiome data.
The analysis is repeated D times, each time using a different taxon as the outcome variable and the remaining taxa as the predictor variables.
We randomly split the data into 3 folds. 
One of the three folds is used as test set whilst the other two folds are used as the train set.
We fit each algorithm on the train set, which is further split into subtrain and validation sets to learn the hyper-parameters of the model. 
We select the best model based on the validation score and fit it on the whole training set. 
We then evaluate it on the test set. 
We repeat this process 3 times and average the test errors to get the overall performance metric.
We show a learned regression model in Figure~\ref{fig:cv}B from cross-validation which is used to infer the co-occurrence network in Figure~\ref{fig:cv}C.
As shown in the network graph where $D=7$ taxa, there is an edge between two taxa only if the relationship between them is positive or negative.

\begin{figure}[t]
\includegraphics[width=\linewidth]{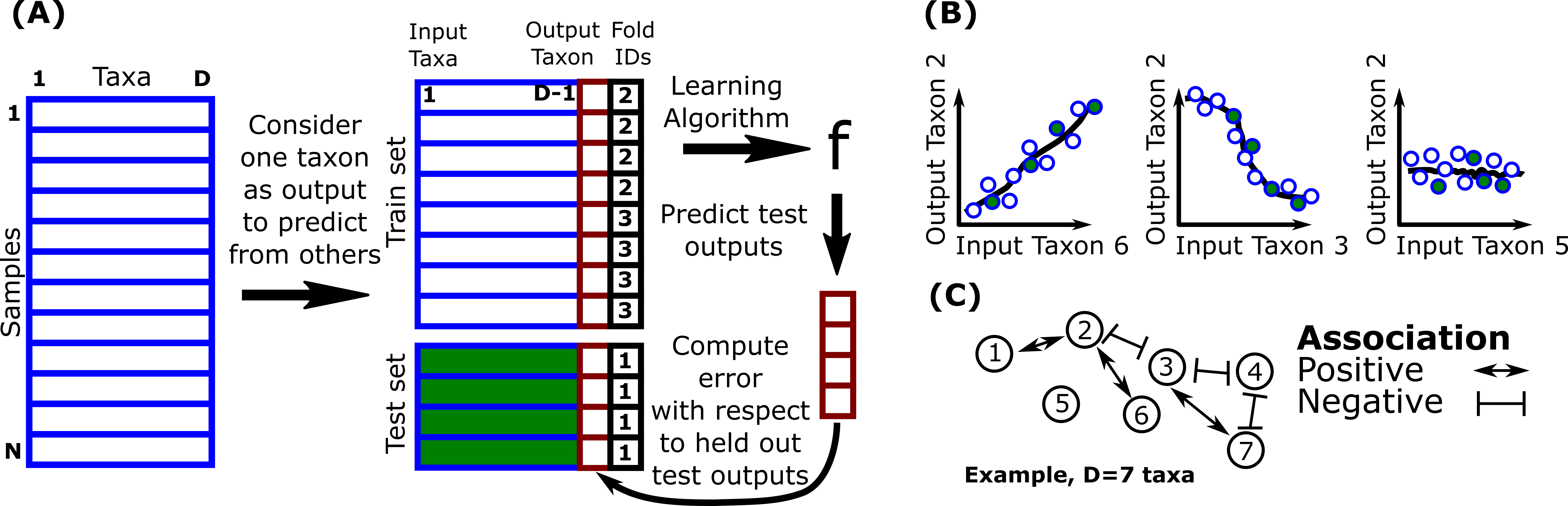}
\vskip 0.2in
\caption{
(A) Proposed cross-validation for evaluating network inference algorithms. 
(B) Learned regression model. 
(C) Co-occurrence network: Nodes represent distinct taxa/bacteria and edges represent positive or negative associations.}
\label{fig:cv}
\end{figure}

\subsection*{Correlation Based Methods}
\label{sec:pearson-spearman-correlation}
\subsubsection*{Pearson Correlation}
Pearson correlation coefficient is the standard tool to infer a network through correlation analysis among all pairs of OTU (Operational Taxonomic Unit) samples. 
It measures the strength and direction of the relationship between two variables. 
It ranges from –1 to 1, where -1 indicates a perfect negative linear relationships, 0 indicates no linear relationship and +1 indicates a perfect positive relationship.  
In most literature~\cite{Deng2012,10.1371/journal.pcbi.1002687, 10.12688/f1000research.9050.2}, there is an arbitrary or pre-determined threshold chosen to select the range of values which is regarded as proof of positive or negative association.
For a pair $(x_1,x_2)$ of standard scaled taxa that follow a bivariate normal distribution with Pearson correlation coefficient $\rho_{x_1,x_2}$, marginal standard deviations $\sigma_{x_1}$ and $\sigma_{x_2}$, the predicted value of $x_1$ given $x_2$ is given below.
\begin{equation}
x_1 =  \rho_{x_1,x_2}\frac{\sigma_{x_1}}{\sigma_{x_2}}(x_2)
\label{eq:pearson_algo}
\end{equation}
This expression is used to compute a prediction for the test set given a trained model that was fit on a training set. 
The parameter, $\rho_{x_1,x_2}$ is learned from the training data.

\subsubsection*{Spearman Correlation}
Spearman correlation coefficient is another popular correlation method for microbial network inference. 
It is often adopted as an alternative to the Pearson correlation coefficient when dealing with non-linear relationships between taxa. 
It is less sensitive and robust to outliers. 
Just as Pearson coefficient, the value of the Spearman coefficient ranges from -1 to +1 , with -1 indicating a perfect negative monotonic relationship, 0 indicating no monotonic relationship, and +1 indicating a perfect positive monotonic relationship.
Spearman coefficient is the Pearson coefficient of ranked data~\cite{ALJABERY20207}. 
We implemented the Spearman algorithm by converting the data into ranks adopting the Pearson Correlation algorithm to predict the ranks. 
For a pair $(r(x_1), r(x_2))$ of standard scaled taxa that follow a bivariate normal distribution with correlation coefficient $\rho_{x_1,x_2}$, marginal standard deviations $\sigma_{r(x_1)}$ and $\sigma_{r(x_2)}$, the predicted rank value, $r(x_1)$ given $r(x_2)$ is given below.
\begin{equation}
r(x_1) = s_{x_1,x_2}\frac{\sigma_{r(x_1)}}{\sigma_{r(x_2)}}(r(x_2))
\label{eq:spearman_algo}
\end{equation}
The model contains the parameters $\sigma_{r(X)}$, $\sigma_{r(Y)}$ and $s_{x_1,x_2}$, which are estimated from the training data. 
We used linear interpolation~\cite{oravkin2021optimal} to infer the actual predicted values from the predicted ranks. 
Linear interpolation is a technique widely adopted to estimate a value within a range of known values by calculating the proportionate relationship between the known values. 
Therefore, we utilize the actual values of the training data alongside their corresponding ranks to estimate the real values of the predicted ranks for the test data. 
Specifically, we use the known pairs of (value, rank) in the training data to form a linear relationship between values and their ranks. 
We then apply this relationship to the predicted ranks to estimate the actual values. 

\begin{figure}[t]
\includegraphics[scale=0.45]{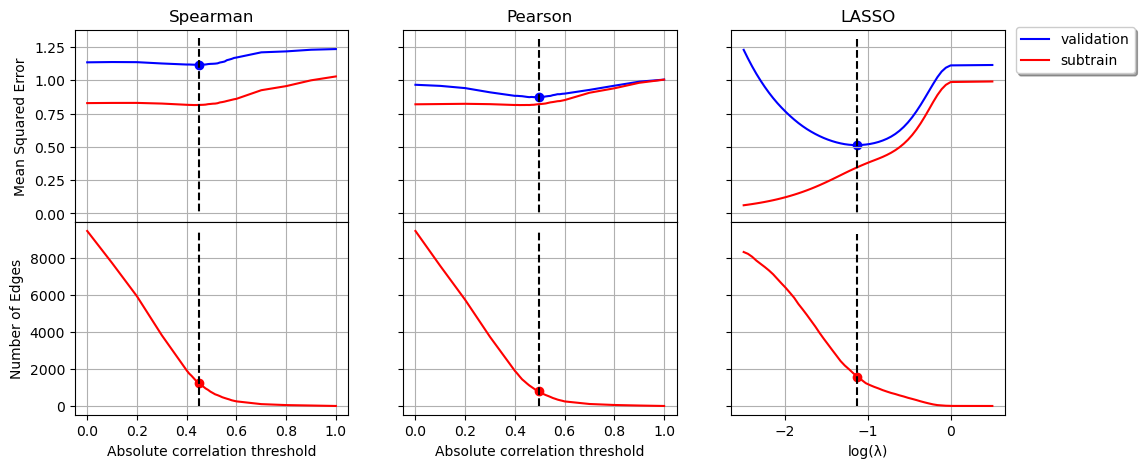}
\caption{Training the Pearson correlation threshold using cross-validation}
\label{fig:algos_mc}
\end{figure}

\subsubsection*{Training of Correlation Based Methods}
One of the most notable challenges is selecting the Pearson or Spearman's rank correlation coefficient threshold for the co-occurrence network inference. 
This should be done so as to limit the network to only edges whose magnitudes are greater than the threshold. 
While most literature~\cite{Deng2012,10.1371/journal.pcbi.1002687, 10.12688/f1000research.9050.2} often choose an arbitrary or pre-determined value as the correlation threshold, the choice of threshold can significantly impact the results and conclusions drawn from the analysis. 
Therefore, it is crucial to carefully consider and justify the choice of correlation threshold. 
The LASSO is a form of linear regression which uses L1 regularization technique and taxon selection to increase the accuracy of prediction~\cite{https://doi.org/10.1111/j.2517-6161.1996.tb02080.x}. 
L1 regularization adds a penalty which causes the regression coefficient of the less contributing taxon to shrink to zero or near zero. 
In this algorithm, the overall objective is to minimize the loss function with respect to the coefficients. 
Let $\mathbf{X} \in \mathbb{R}^{N \times D}$ be the compositional data matrix where each row represents a sample and each column represents a taxon, 
$\mathbf{y} \in \mathbb{R}^n$ be the target taxon vector, 
$\mathbf{w} \in \mathbb{R}^p$ be the coefficient vector, and 
$\beta_0$ be the intercept term.
The linear model can be defined as:
\begin{equation}
f(\mathbf{x}) = \beta_0 + \mathbf{x}^T \mathbf{w}
\label{eq:lasso}
\end{equation}
Then, the loss function of the LASSO regression model can be formulated as:
\begin{equation*}
L(\beta_0,\mathbf{w}) = \frac{1}{2n}||\mathbf{y} - \beta_0 - \mathbf{X}\mathbf{w}||^2_2 + \lambda ||\mathbf{w}||_1
\end{equation*}
The first term is the residual sum of squares (RSS), which is the deviation of the predicted values from the actual values. 
The second term is the L1 penalty term that encourages sparsity in the coefficient estimates.  
$\lambda$ is the regularization parameter that controls the amount of shrinkage. 
With cross-validation algorithm, optimum LASSO model is selected and the coefficient of this model is used for network inference~\cite{scikit-learn}. 
Train set is split into subtrain set (used to learn regression coefficients) and validation set (used to learn the degree of L1 regularization, which controls sparsity / number of edges in co-occurrence network).
In Figure~\ref{fig:algos_mc}, we leverage 3-fold cross-validation to choose the optimal values for the correlation coefficient threshold and $\lambda$, which minimize the validation error when used for prediction.
This figure illustrates how the test error and the number of edges vary with correlation threshold for the correlation based algorithms and $\log(\lambda)$ of the LASSO model, applied to the Amgut2 data set~\cite{mcdonald2018american}.
We systematically varied these hyperparameters and monitored the resulting subtrain and validation errors. The adoption of $\log(\lambda)$, rather than $\lambda$, enhances the interpretability of the graph and mitigates potential distortion arising from extreme $\lambda$ values.
The error curves reveal tendencies towards overfitting for small thresholds or $\log(\lambda)$ values (leading to many edges) and underfitting for large thresholds or $\log(\lambda)$ values (resulting in fewer edges). 
We selected the value of $\lambda$ corresponding to the minimum validation error, which yielded a network with 1585 edges. 
For the Pearson correlation coefficient, the optimal threshold was found to be $0.495$, resulting in 785 edges, while for the Spearman correlation coefficient, the optimal threshold was $0.448$, resulting in 1231 edges. 
These thresholds were chosen because they minimized the validation error, rendering correlation values smaller than these thresholds incapable of establishing edges in the co-occurrence network.

\subsection*{Gaussian Graphical Model (GGM)}
\label{sec:gaussian-graphical-model}
The Gaussian distribution is a continuous and symmetrical probability distribution that explains how the outcomes of a random variable are distributed. 
The shape of the Gaussian distribution is determined by its mean and standard deviation, which evaluates the location and spread of the distribution, respectively. 
Most observations cluster around the mean of the distribution~\cite{gaussian-dist-article}.
The Probability Density Function (PDF) of a multivariate normal distribution is frequently employed in data analysis to model complex data sets that involve multiple variables~\cite{scott2015multivariate}. 
Let $D$ be the total number of taxa, 
$\boldsymbol{x}$ be a $D$-dimensional row/sample vector, 
$\boldsymbol{\Sigma}$ be a $D\times D$ covariance matrix, 
$\boldsymbol{\Omega}$ be a $D\times D$ precision matrix comprised of $\omega_{ij}$ elements and 
$\boldsymbol{x}^T$ denote the transpose of $\boldsymbol{x}$. 
The multivariate normal distribution PDF is given below.
\begin{equation*}
    f(\boldsymbol{x}) = \frac{1}{\sqrt{(2\pi)^D|\boldsymbol{\Sigma}|}}\exp\left(-\frac{1}{2}\boldsymbol{x}^T\boldsymbol{\Omega}\boldsymbol{x}\right)
\end{equation*}
The predicted value of the first taxon ($x_1$) can be calculated by finding the conditional  mean of the distribution. 
This is the value of $x_1$ when $f(x)$ is maximum. 
Therefore, we take the partial derivative of $f(x)$ with respect to $x_1$ and equate to zero. 
As demonstrated in the \nameref{GGM_Proof}, we solve for the value of $x_1$, which leads us to the following equation which we use to compute predictions,
\begin{equation}
x_1 = \frac{-1}{2\omega_{11}}\left( \sum_{i=2}^D \omega_{i1} x_i + \sum_{j=2}^D \omega_{1j} x_j \right) 
\label{eq:ggm}
\end{equation}
This is well known for the special case of $D=2$ (See proof in Supplementary Information), the conditional mean of a bivariate normal (\ref{eq:pearson_algo}) under the assumption that data is standard scaled thus zero mean and unit variance. 
Our contribution here is to derive a formula for the general case, $D>2$ (See proof in Supplementary Information).
The inverse covariance matrix (precision matrix) is computed from the train dataset in the GGM. 
The conditional independence structure among taxa is represented by the sparsity pattern of the precision matrix. 
This sparsity pattern can be estimated from data using various methods, such as maximum likelihood estimation or penalized likelihood methods. 
The Graphical Lasso (GLASSO) is used to estimate the precision matrix from high-dimensional data. 
In GLASSO, the penalty is applied to the elements of the precision matrix, resulting in a sparse estimate of the matrix. 
Given a train data matrix $\mathbf{X} \in \mathbb{R}^{N \times D}$ where $N$ is the number of samples and $D$ is the number of taxa, the goal is to estimate the precision matrix $\bm{\Theta}$ that satisfies the following optimization problem. 
Let $\mathbf{S} = \frac{1}{N} \mathbf{X}^\top \mathbf{X}$ be the sample covariance matrix, 
$\lVert \bm{\Theta} \rVert_1$ be the L1-norm penalty term to promote sparsity in the precision matrix, 
$\lambda$ be the regularization parameter that controls the strength of the penalty term.
The precision matrix is given by:  
\begin{equation*}
\hat{\bm{\Theta}} = \underset{\bm{\Theta} \succeq 0}{\operatorname{argmin}} \left( \operatorname{tr}(\mathbf{S}\bm{\Theta}) - \log\det(\bm{\Theta}) + \lambda \lVert \bm{\Theta} \rVert_1 \right)
\label{eq:precision}
\end{equation*}
The constraint $\bm{\Theta} \succeq 0$ enforces the positive semi-definiteness of the precision matrix. 
The solution $\hat{\bm{\Theta}}$ corresponds to the maximum likelihood estimate of the precision matrix under the sparsity constraint. 
The precision matrix is used to infer the network graph of the taxa based on their conditional dependencies. The presence or absence of an edge between taxa $i$ and $j$ in the graph is determined by the value of $\hat{\Theta}_{ij}$ in the precision matrix. 
An edge between taxa $i$ and $j$ exists if and only if $\hat{\Theta}_{ij} \neq 0$.

\section*{Results}
In this study, we conducted a real microbiome composition data analysis on Amgut1~\cite{10.1371/journal.pcbi.1004226}, crohns~\cite{10.1093/bioinformatics/btz824} and iOral~\cite{Ha2020} real microbiome composition data sets because they are public and widely used when comparing previous algorithms. 
We used the Yeo-Johnson power transformation~\cite{10.1093/biomet/87.4.954} in combination with standard scaling, so that each column has zero mean and unit variance (for numerical stability). 
We wrote python code to implement the various algorithms. 
We specifically utilized the LassoCV and GraphicalLassoCV classes from scikit-learn package~\cite{scikit-learn} to implement the LASSO algorithm and estimate the precision matrix for the GGM respectively.

\subsection*{Impact of Total Sample Size on Test Error}
In Figure~\ref{fig:model_comparison}, we investigated how the test error varies with the number of total samples for the various algorithms, including the featureless/baseline method, where the predicted values were computed using the mean of the train data.
We sub-sample each data set by randomly dividing it into a series of different sample sizes (10, 20, etc), before we run the cross-validation analysis on each sample size. 
The relationship between each pair of taxon columns is utilized for prediction, as shown in the equations~\ref{eq:pearson_algo} and~\ref{eq:spearman_algo} for Pearson and Spearman algorithms respectively.
The predicted value for both LASSO and GGM algorithms is calculated using the equations~\ref{eq:lasso} and~\ref{eq:ggm} respectively.
The test error was computed by taking the average of the Mean Squared Error (MSE) of the predicted values compared to the actual test values, across all the taxa in each of the data sets, test sets and sub-samples. 
The lower and upper bounds of the MSE line represent the variance of the MSE. 
For the Amgut1 data set, GGM achieved the highest accuracy from 10 to 20 sample size, but LASSO performed best for larger sample sizes (above 30). 
The GGM outperformed the other algorithms on the iOral data set.
The results from the crohns data set suggest that both LASSO and GGM algorithms may be good choices for this data set, as they performed similarly well. 
The figure also provides insights into the minimum sample size required for a useful cross-validation of the algorithms. 
The plot reveals that significant differences between algorithms are apparent even with only 10 samples.
These findings highlight the importance of selecting an appropriate algorithm for a given dataset, as different algorithms may perform differently depending on the characteristics of the data. 
Therefore, it is crucial to consider multiple algorithms and evaluate their performance before selecting the most appropriate one. 
Additionally, it may be necessary to use a combination of algorithms to obtain the best results.

\begin{figure}[t]
    \includegraphics[scale=0.8]{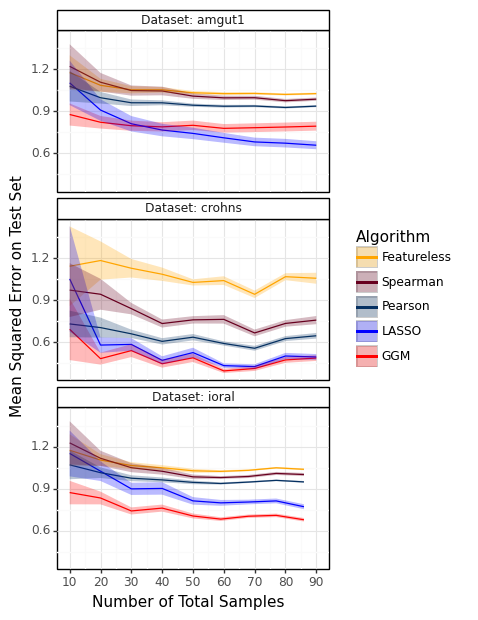}
    \caption{Model Comparison using Test Error}
\label{fig:model_comparison}
\end{figure}

\subsection*{Microbial Association Network Inference}
The microbial association network was inferred using the various algorithms on the three real microbiome datasets: amgut1, crohns, and ioral. 
Following the identification of the optimal model for each algorithm, pairwise positive and negative associations between taxa in the data sets were computed to infer the co-occurrence network.
For the correlation-based algorithms, the correlation matrix is estimated by calculating the pairwise correlation coefficient for each taxon pair, and the network is constrained by the learned correlation threshold.
In the case of the LASSO algorithm, we save the coefficients of the optimal model at each iteration of the taxa columns, hence forming an association matrix.
For the GGM, the GLASSO inferred precision matrix is used for the association matrix.
We compute the mean of the upper and lower triangular matrices for each of the LASSO and the GGM, resulting in lower triangular matrices for each algorithm.
In the resultant lower triangular matrix of the association matrix, an edge is identified if its value is non-zero. 
A positive value indicates a positive association, while a negative value indicates a negative association.
Through the application of 3-fold cross-validation analysis, three networks are inferred for each algorithm based on the fold IDs.
The final network obtained is the median of the three networks inferred by the 3 folds.

\begin{figure}[t]
    \includegraphics[scale=0.3]{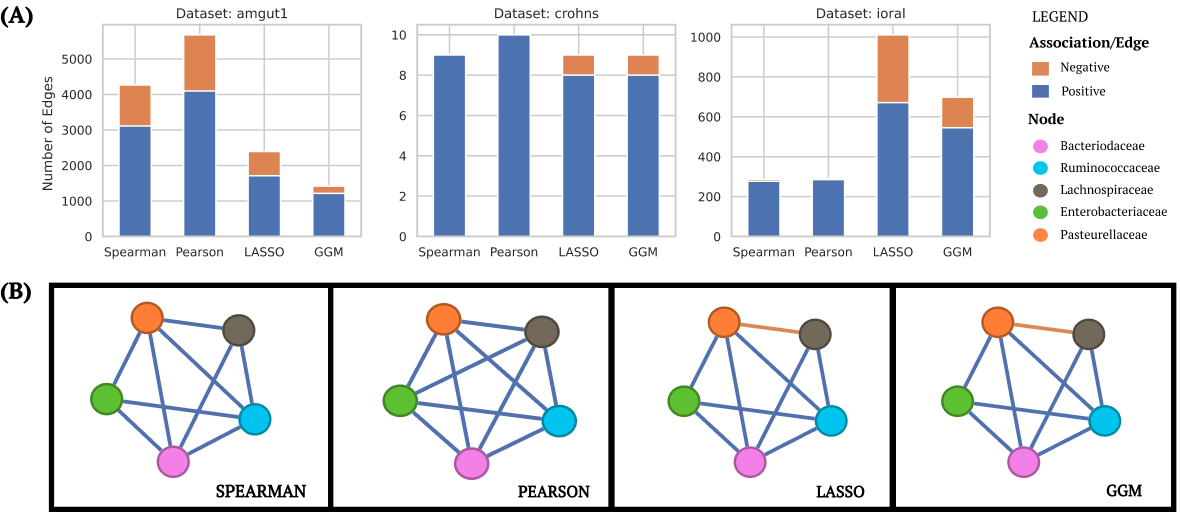}
    \vskip 0.3in
    \caption{
    (A) Model comparison using inferred positive and negative associations.\\ 
    (B) Microbial Network graph of crohns data set.\\
    }
\label{fig:model_ass_comparison}
\end{figure}

In Figure~\ref{fig:model_ass_comparison}A, we showcase how the number of edges inferred varies with each of the algorithms in the various real microbiome data sets.
From the graph, we can infer that the Pearson mostly has more number of edges than Spearman because the taxa have strong linear relationships with each other and moderate non-linear relationships that are not captured by linear correlation.
The LASSO typically infers more number of edges than GGM because GGM employs the precision matrix that measures the partial correlation between taxa.

In Figure~\ref{fig:model_ass_comparison}B, we show the co-occurrence network graph of the crohns data set which has 5 different bacteria and 100 total samples.
The Pearson algorithm produces a full connected network, while the Spearman algorithm excludes the edge between \textit{Lachnospiraceae} and \textit{Enterobacteriaceae}.
The LASSO and the GGM share the same network topology as the Spearman, but differ in the sign of the associations: the LASSO and GGM reveal a negative association between \textit{Pasteurellaceae} and \textit{Lachnospiraceae}.
The choice of algorithm depends on the research question and the data quality. 
If the goal is to find a robust network structure with high confidence, then LASSO or GGM may be preferable. 
However, if the goal is to explore a comprehensive and diverse network structure with low confidence, then Pearson or Spearman may be more suitable.

\section*{Discussion and Conclusion}
This study provides a comparative analysis of different algorithms for inferring microbial association networks from real microbiome composition data.
We propose cross-validation as a more widely applicable evaluation criterion for training and testing the various algorithms used for inferring microbial co-occurrence networks. 
We also introduce a novel technique of using previous algorithms for prediction on test data. 
The results from Figure~\ref{fig:model_comparison} show that the accuracy of the algorithms generally improved upon increasing the number of total samples. 
We can also infer that the selection of an algorithm should depend on the specific data set being examined and the research question being answered since the choice of algorithm can have a significant impact on the structure of the resulting microbial network. 
The findings indicate that LASSO and GGM are most accurate for inferring co-occurrence networks, in the Amgut1, crohns and iOral real microbiome composition data sets that we examined. 
For future research, we are considering using~(\ref{eq:ggm}) for training GGMs as well as testing which we would expect to be more accurate than employing the maximum likelihood approach to estimate the precision matrix. 
We are also interested in generalizing our proposed cross-validation methods to more complex data with several qSIP features like the abundance, growth rate, death rate and carbon uptake of micro-organisms~\cite{10.1371/journal.pone.0189616}
The findings indicate that both LASSO and GGM are dependable and effective for inferring co-occurrence networks.

\section*{Acknowledgments}
We express our immense gratitude to colleagues Kyle Rust, Jadon Fowler and Cameron Bodine for their expert feedback, constructive criticisms and suggestions which were instrumental to the completion of this research.
This work was funded by the National Science Foundation grant \#2125088 from Rules of Life Program.

\section*{Supporting Information}
\paragraph*{Data and Code Availability.}
\label{S1_Data}
The data, figures and code used in this study is publicly available and can be accessed at \url{https://github.com/EngineerDanny/CS685-Microbe-Network-Research}.
To reproduce the results, make sure the required libraries are installed and the required data is imported.

\paragraph*{Gaussian Graphical Model Proof}
\label{GGM_Proof}
Let $D$ be the total number of taxa, $\boldsymbol{x}$ be the $D$-dimensional row/sample vector, which we assume follows a multivariate normal distribution with mean 0 and $D\times D$ covariance matrix,
$\boldsymbol{\Sigma}$ of size $D\times D$. 
Let $\boldsymbol{\Omega}$ be the $D\times D$ precision matrix comprised of $\omega_{ij}$ elements and $\boldsymbol{x}^T$ denote the transpose of $\boldsymbol{x}$.
\begin{proposition}
Equation~(\ref{eq:ggm}) can be used to compute the conditional mean of one taxon $x_1$ in a Gaussian Graphical Model, given the other taxa $x_2,\dots,x_D$.
\end{proposition}
\begin{proof}

The multivariate normal distribution Probability Density Function is given by:
\begin{equation*}
    f(\boldsymbol{x}) = \frac{1}{\sqrt{(2\pi)^D|\boldsymbol{\Sigma}|}}\exp\left(-\frac{1}{2}\boldsymbol{x}^T\boldsymbol{\Omega}\boldsymbol{x}\right)
\end{equation*}

The predicted value of taxon $x_1$ can be calculated by finding the conditional mean of the distribution via the equation below. 
\begin{align*}
\frac{\partial}{\partial x_1}f(\boldsymbol{x}) &= 0 \\
\frac{\partial}{\partial x_1} \left(\frac{1}{\sqrt{(2\pi)^D|\boldsymbol{\Sigma}|}}\exp\left(-\frac{1}{2}\boldsymbol{x}^T\boldsymbol\Omega\boldsymbol{x}\right)\right) &= &&\text{( Substitute $f(\boldsymbol{x})$ )}\\
\frac{\partial}{\partial x_1} \exp\left(-\frac{1}{2}\boldsymbol{x}^T\boldsymbol{\Omega}\boldsymbol{x}\right) &= &&\text{( Remove constants )}\\
 \exp\left(-\frac{1}{2}\boldsymbol{x}^T\boldsymbol{\Omega}\boldsymbol{x}\right)
 \frac{\partial}{\partial x_1}\left(-\frac{1}{2}\boldsymbol{x}^T\boldsymbol{\Omega}\boldsymbol{x}\right) &= &&\text{( Differentiate expression )}\\
  \frac{\partial}{\partial x_1}\left(\boldsymbol{x}^T\boldsymbol{\Omega}\boldsymbol{x}\right) &= &&\text{( Remove constants and simplify )}\\
  \frac{\partial}{\partial x_1} \left(\sum_{i=1}^D \sum_{j=1}^D \omega_{ij} x_i x_j\right) &= &&\text{( Expand the bracket )}\\
  \frac{\partial}{\partial x_1} \left(\omega_{11} x_1^2 + \sum_{i=2}^D \omega_{i1} x_i x_1 + \sum_{j=2}^D \omega_{1j} x_1 x_j\right) &= &&\text{( Simplify further )}\\
2\omega_{11} x_1 + \sum_{i=2}^D \omega_{i1} x_i + \sum_{j=2}^D \omega_{1j} x_j &= 0 &&\text{( Differentiate expression )}\\\\
\implies \frac{-1}{2\omega_{11}}\left( \sum_{i=2}^D \omega_{i1} x_i + \sum_{j=2}^D \omega_{1j} x_j \right) &= x_1 &&\text{( Solve for $x_1$ )}\\ 
\end{align*}
\end{proof}

\begin{corollary}
For the special case of $D=2$, the multivariate Gaussian conditional mean~(\ref{eq:ggm}) simplifies to~(\ref{eq:pearson_algo}) under the assumption that the data is standard scaled; zero mean and unit variance for each taxon.
\end{corollary}
\begin{proof}
For a pair $(x_1,x_2)$ of taxa that follow a bivariate normal distribution with correlation coefficient $\rho_{x_1,x_2}$, marginal standard deviations $\sigma_{x_1}$ and $\sigma_{x_2}$, the predicted value of $x_1$ given $x_2$ is given below.
\begin{align*}
\boldsymbol{\Omega} &= \frac{1}{(1-\rho_{x_1,x_2}^2)}
\begin{bmatrix}
\frac{\sigma_{x_2}^2}{\sigma_{x_1}^2\sigma_{x_2}^2} & -\frac{\rho_{x_1,x_2}}{\sigma_{x_1}\sigma_{x_2}} \\
-\frac{\rho_{x_1,x_2}}{\sigma_{x_1}\sigma_{x_2}} & \frac{\sigma_{x_1}^2}{\sigma_{x_1}^2\sigma_{x_2}^2}
\end{bmatrix} &&\text{( Define $\boldsymbol{\Omega}$, when $D = 2$ )}\\
&= \begin{bmatrix}
\frac{1}{\sigma_{x_1}^2(1-\rho_{x_1,x_2}^2)} & -\frac{\rho_{x_1,x_2}}{\sigma_{x_1}\sigma_{x_2}(1-\rho_{x_1,x_2}^2)} \\
-\frac{\rho_{x_1,x_2}}{\sigma_{x_1}\sigma_{x_2}(1-\rho_{x_1,x_2}^2)} & \frac{1}{\sigma_{x_2}^2(1-\rho_{x_1,x_2}^2)}
\end{bmatrix} &&\text{( Expand and simplify )}\\
x_1 &= \frac{-1}{2\omega_{11}}\left( \sum_{i=2}^2 \omega_{i1} x_i + \sum_{j=2}^2 \omega_{1j} x_j \right) &&\text{( Equation for $x_1$, when $D=2$ )}\\
&= \frac{-1}{2\omega_{11}}\left( \omega_{21} x_2 + \omega_{12} x_2 \right) &&\text{( Simplify )}\\
&= \frac{-1}{\omega_{11}}\left(\omega_{21} x_2 \right) &&\text{( Substitute $\omega_{21} = \omega_{12}$ from $\boldsymbol{\Omega}$ ) }\\
&= \frac{-1}{\left( \frac{1}{\sigma_{x_1}^2(1-\rho_{x_1,x_2}^2)} \right)}\left( -\frac{\rho_{x_1,x_2}}{\sigma_{x_1}\sigma_{x_2}(1-\rho_{x_1,x_2}^2)} x_2 \right) &&\text{( Substitute the values of $\omega$ ) }\\
\implies x_1 &= \rho_{x_1,x_2}\frac{\sigma_{x_1}}{\sigma_{x_2}}(x_2) && \text{( Simplify ) } 
\end{align*}
\end{proof}


%
%
%

\bibliographystyle{plos2015}
\bibliography{references}

\end{document}